\PassOptionsToPackage{obeyspaces}{url}
\documentclass[sigconf,screen]{acmart}

\usepackage{bm}
\usepackage{graphicx}
\usepackage{amsmath,amsfonts,amsthm}
\usepackage{booktabs}
\usepackage{multirow}
\usepackage{makecell}
\usepackage{textcomp}
\usepackage{array}
\usepackage{textcomp}
\usepackage{xcolor}
\usepackage[font=footnotesize]{subfig}
\usepackage{multirow}
\usepackage{float}
\usepackage{xspace}
\usepackage[export]{adjustbox}
\usepackage{nicefrac}
\usepackage{colortbl}
\usepackage{tablefootnote}
\usepackage[referable]{threeparttablex}

\usepackage[linesnumbered,ruled,vlined]{algorithm2e}
\SetKwRepeat{Do}{do}{while}
\DontPrintSemicolon

\SetCommentSty{commentsty}

\usepackage{mathtools}

\DeclarePairedDelimiterX{\infdivx}[2]{(}{)}{%
  #1\;\delimsize\|\;#2%
}

\newtheorem{theorem}{Theorem}[section]

\theoremstyle{remark}
\newtheorem*{remark}{Remark}

\DeclareMathOperator{\GC}{GC}

\copyrightyear{2021}
\acmYear{2021}
\setcopyright{iw3c2w3}
\acmConference[WWW '21]{Proceedings of the Web Conference 2021}{April 19--23, 2021}{Ljubljana, Slovenia} 
\acmBooktitle{Proceedings of the Web Conference 2021 (WWW '21), April 19--23, 2021, Ljubljana, Slovenia}
\acmPrice{}
\acmDOI{10.1145/3442381.3449802}
\acmISBN{978-1-4503-8312-7/21/04}


\author[Yanqiao Zhu, Yichen Xu, Feng Yu, Qiang Liu, Shu Wu, and Liang Wang]{Yanqiao Zhu$^{1,2,}$*, Yichen Xu$^{3,}$*, Feng Yu$^{4}$, Qiang Liu$^{1,2}$, Shu Wu$^{1,2,\dagger}$, and Liang Wang$^{1,2}$}

\makeatletter
\def\authornotetext#1{
	\g@addto@macro\@authornotes{%
	\stepcounter{footnote}\footnotetext{#1}}%
}
\makeatother

\authornotetext{The first two authors made equal contribution to this work.}
\authornotetext{To whom correspondence should be addressed.}

\affiliation{%
	\institution{$^1$Center for Research on Intelligent Perception and Computing, Institute of Automation, Chinese Academy of Sciences}
	\institution{$^2$School of Artificial Intelligence, University of Chinese Academy of Sciences}
	\institution{$^3$School of Computer Science, Beijing University of Posts and Telecommunications \quad $^4$Alibaba Group}
	\country{}
}

\email{yanqiao.zhu@cripac.ia.ac.cn,  linyxus@bupt.edu.cn}
\email{yf271406@alibaba-inc.com,  {qiang.liu, shu.wu, wangliang}@nlpr.ia.ac.cn}

\def\authors{Yanqiao Zhu, Yichen Xu, Feng Yu, Qiang Liu, Shu Wu, and Liang Wang}

\begin{document}

\newcommand{\themodel}{GCA\xspace}

\title{Graph Contrastive Learning with Adaptive Augmentation}

\begin{abstract}
Recently, contrastive learning (CL) has emerged as a successful method for unsupervised graph representation learning. Most graph CL methods first perform stochastic augmentation on the input graph to obtain two graph views and maximize the agreement of representations in the two views.
Despite the prosperous development of graph CL methods, the design of graph augmentation schemes---a crucial component in CL---remains rarely explored. We argue that the data augmentation schemes should preserve intrinsic structures and attributes of graphs, which will force the model to learn representations that are insensitive to perturbation on unimportant nodes and edges.
However, most existing methods adopt uniform data augmentation schemes, like uniformly dropping edges and uniformly shuffling features, leading to suboptimal performance.
In this paper, we propose a novel graph contrastive representation learning method with adaptive augmentation that incorporates various priors for topological and semantic aspects of the graph.
Specifically, on the topology level, we design augmentation schemes based on node centrality measures to highlight important connective structures. On the node attribute level, we corrupt node features by adding more noise to unimportant node features, to enforce the model to recognize underlying semantic information.
We perform extensive experiments of node classification on a variety of real-world datasets. Experimental results demonstrate that our proposed method consistently outperforms existing state-of-the-art baselines and even surpasses some supervised counterparts, which validates the effectiveness of the proposed contrastive framework with adaptive augmentation.
\end{abstract}


\maketitle

\section{Introduction}

\begin{figure*}
	\centering
	\includegraphics[width=\linewidth]{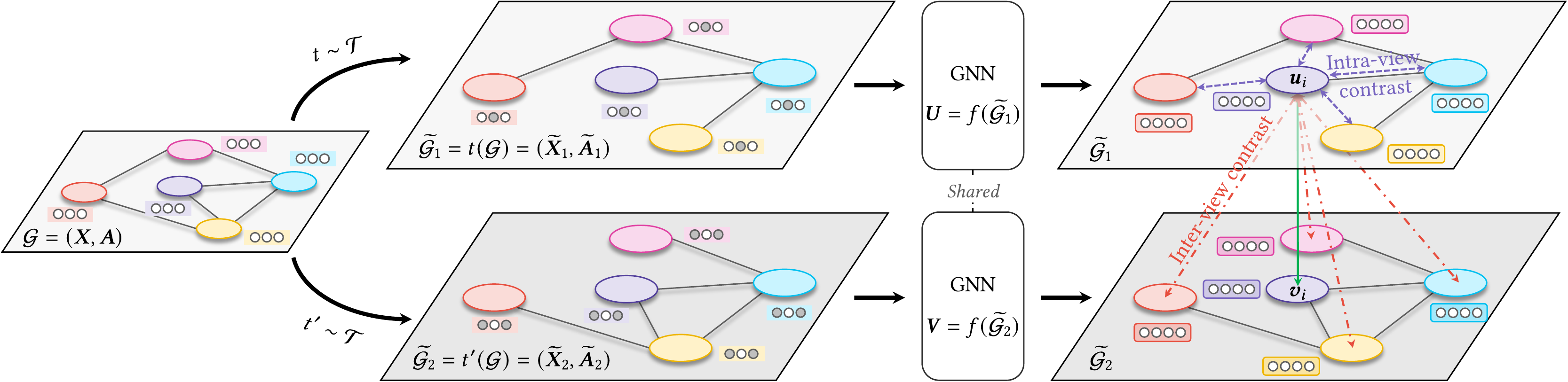}
	\caption{Our proposed deep Graph Contrastive representation learning with Adaptive augmentation (\themodel) model. We first generate two graph views via stochastic augmentation that is adaptive to the graph structure and attributes. Then, the two graphs are fed into a shared Graph Neural Network (GNN) to learn representations. We train the model with a contrastive objective, which pulls representations of one node together while pushing node representations away from other node representations in the two views. N.B., we define the negative samples as all other nodes in the two views. Therefore, negative samples are from two sources, intra-view (in purple) and inter-view nodes (in red).}
	\label{fig:model}
\end{figure*}

Over the past few years, graph representation learning has emerged as a powerful strategy for analyzing graph-structured data. Graph representation learning using Graph Neural Networks (GNN) has received considerable attention, which aims to transform nodes to low-dimensional dense embeddings that preserve graph attributive and structural features.
However, existing GNN models are mostly established in a supervised manner \cite{Kipf:2016tc,Velickovic:2018we,Hu:2019vq}, which require abundant labeled nodes for training.
Recently, Contrastive Learning (CL), as revitalization of the classical Information Maximization (InfoMax) principle \cite{Linsker:1988ho}, achieves great success in many fields, e.g., visual representation learning \cite{Tian:2019vw,He:2020tu,Bachman:2019wp} and natural language processing \cite{Weston:2008kg,Kavukcuoglu:2013to}. These CL methods seek to maximize the Mutual Information (MI) between the input (i.e. images) and its representations (i.e. image embeddings) by contrasting positive pairs with negative-sampled counterparts.

Inspired by previous CL methods, Deep Graph InfoMax (DGI) \cite{Velickovic:2019tu} marries the power of GNN into InfoMax-based methods. DGI firstly augments the original graph by simply shuffling node features. Then, a contrastive objective is proposed to maximize the MI between node embeddings and a global summary embedding.
Following DGI, GMI \cite{Peng:2020gw} proposes two contrastive objectives to directly measure MI between input and representations of nodes and edges respectively, without explicit data augmentation.
Moreover, to supplement the input graph with more global information, MVGRL \cite{Hassani:2020un} proposes to augment the input graph via graph diffusion kernels \cite{Klicpera:2019vc}. Then, it constructs graph views by uniformly sampling subgraphs and learns to contrast node representations to global embeddings across the two views.

Despite the prosperous development of graph CL methods, data augmentation schemes, proved to be a critical component for visual representation learning \cite{Wu:2020tj}, remain rarely explored in existing literature.
Unlike abundant data transformation techniques available for images and texts, graph augmentation schemes are non-trivial to define in CL methods, since graphs are far more complex due to the non-Euclidean property. 
We argue that the augmentation schemes used in the aforementioned methods suffer from two drawbacks.
At first, simple data augmentation in \emph{either} the structural domain \emph{or} the attribute domain, such as feature shifting in DGI \cite{Velickovic:2019tu}, is not sufficient for generating diverse neighborhoods (i.e. contexts) for nodes, especially when node features are sparse, leading to difficulty in optimizing the contrastive objective.
Secondly, previous work ignores the discrepancy in the impact of nodes and edges when performing data augmentation. For example, if we construct graph views by \emph{uniformly} dropping edges, removing some influential edges will deteriorate the embedding quality.
As the representations learned by the contrastive objective tend to be \emph{invariant} to corruption induced by the data augmentation scheme \cite{Xiao:2020vt}, the data augmentation strategies should be \emph{adaptive} to the input graph to reflect its intrinsic patterns.
Again, taking the edge removing scheme as an example, we can give larger probabilities to unimportant edges and lower probabilities to important ones, when randomly removing the edges. Then, this scheme is able to guide the model to ignore the introduced noise on unimportant edges and thus learn important patterns underneath the input graph.

To this end, we propose a novel contrastive framework for unsupervised graph representation learning, as shown in Figure \ref{fig:model}, which we refer to as \underline{G}raph \underline{C}ontrastive learning with \underline{A}daptive augmentation, \themodel for brevity.
In \themodel, we first generate two correlated graph views by performing stochastic corruption on the input. Then, we train the model using a contrastive loss to maximize the agreement between node embeddings in these two views.
Specifically, we propose a joint, adaptive data augmentation scheme at both topology and node attribute levels, namely removing edges and masking features, to provide diverse contexts for nodes in different views, so as to boost optimization of the contrastive objective.
Moreover, we identify important edges and feature dimensions via centrality measures.
Then, on the topology level, we adaptively drop edges by giving large removal probabilities to unimportant edges to highlight important connective structures.
On the node attribute level, we corrupt attributes by adding more noise to unimportant feature dimensions, to enforce the model to recognize underlying semantic information.

The core contribution of this paper is two-fold:
\begin{itemize}
	\item Firstly, we propose a general contrastive framework for unsupervised graph representation learning with strong, adaptive data augmentation. The proposed \themodel framework jointly performs data augmentation on both topology and attribute levels that are adaptive to the graph structure and attributes, which encourages the model to learn important features from both aspects.
	\item Secondly, we conduct comprehensive empirical studies using five public benchmark datasets on node classification under the commonly-used linear evaluation protocol. \themodel consistently outperforms existing methods and our unsupervised method even surpasses its supervised counterparts on several transductive tasks.
\end{itemize}
To make the results of this work reproducible, we make all the code publicly available at \url{https://github.com/CRIPAC-DIG/GCA}.

The remaining of the paper includes the following sections. We briefly review related work in Section \ref{sec:related-work}. In Section \ref{sec:method}, we present the proposed \themodel model in detail. The results of the experiments are analyzed in Section \ref{sec:experiments}. Finally, we conclude the paper in Section \ref{sec:conclusion}. For readers of interest, additional configurations of experiments and details of proofs are provided in Appendix \ref{appendix:implementation} and \ref{appendix:proofs}, respectively.

\section{Related Work}
\label{sec:related-work}

In this section, we briefly review prior work on contrastive representation learning. Then, we review graph representation learning methods. At last, we provide a summary of comparisons between the proposed method and its related work.

\subsection{Contrastive Representation Learning}
Being popular in self-supervised representation learning, contrastive methods aim to learn discriminative representations by contrasting positive and negative samples. For visual data, negative samples can be generated using a multiple-stage augmentation pipeline \cite{Chen:2020wj,Bachman:2019wp,Falcon:2020uv}, consisting of color jitter, random flip, cropping, resizing, rotation \cite{Gidaris:2018wr}, color distortion \cite{Larsson:2017vt}, etc.
Existing work \cite{Wu:2018kw,Tian:2019vw,He:2020tu} employs a memory bank for storing negative samples. Other work \cite{Bachman:2019wp,Ye:2019we,Chen:2020wj} explores in-batch negative samples. For an image patch as the anchor, these methods usually find a global summary vector \cite{Hjelm:2019uk,Bachman:2019wp} or patches in neighboring views \cite{vandenOord:2018ut,Henaff:2020ta} as the positive sample, and contrast them with negative-sampled counterparts, such as patches of other images within the same batch \cite{Hjelm:2019uk}.

Theoretical analysis sheds light on the reasons behind their success \cite{Poole:2019vk}. Objectives used in these methods can be seen as maximizing a lower bound of MI between input features and their representations \cite{Linsker:1988ho}. However, recent work \cite{Tschannen:2020uo} reveals that downstream performance in evaluating the quality of representations may strongly depend on the bias that is encoded not only in the convolutional architectures but also in the specific estimator of the InfoMax objective.

\subsection{Graph Representation Learning}
Many traditional methods on unsupervised graph representation learning inherently follow the contrastive paradigm \cite{Perozzi:2014ib,Grover:2016ex,Kipf:2016ul,Hamilton:2017wa}.
Prior work on unsupervised graph representation learning focuses on local contrastive patterns, which forces neighboring nodes to have similar embeddings. For example, in the pioneering work DeepWalk \cite{Perozzi:2014ib} and node2vec \cite{Grover:2016ex}, nodes appearing in the same random walk are considered as positive samples. Moreover, to model probabilities of node co-occurrence pairs, many studies resort to Noise-Contrastive Estimation (NCE) \cite{Gutmann:2012eq}.
However, these random-walk-based methods are proved to be equivalent to factorizing some forms of graph proximity (e.g., multiplication of the adjacent matrix to model high-order connection) \cite{Qiu:2018ez} and thus tend to overly emphasize on the encoded structural information. Also, these methods are known to be error-prone with inappropriate hyperparameter tuning \cite{Perozzi:2014ib,Grover:2016ex}.

Recent work on Graph Neural Networks (GNNs) employs more powerful graph convolutional encoders over conventional methods. Among them, considerable literature has grown up around the theme of supervised GNN \cite{Kipf:2016tc,Velickovic:2018we,Hu:2019vq,Wu:2019vz}, which requires labeled datasets that may not be accessible in real-world applications.
Along the other line of development, unsupervised GNNs receive little attention. Representative methods include GraphSAGE \cite{Hamilton:2017tp}, which incorporates DeepWalk-like objectives.
Recent work DGI \cite{Velickovic:2019tu} marries the power of GNN and CL, which focuses on maximizing MI between global graph-level and local node-level embeddings. Specifically, to implement the InfoMax objective, DGI requires an injective readout function to produce the global graph-level embedding. However, it is too restrictive to fulfill the injective property of the graph readout function, such that the graph embedding may be deteriorated. In contrast to DGI, our preliminary work \cite{Zhu:2020vf} proposes to not rely on an explicit graph embedding, but rather focuses on maximizing the agreement of node embeddings across two corrupted views of the graph.

Following DGI, GMI \cite{Peng:2020gw} employs two discriminators to directly measure MI between input and representations of both nodes and edges without data augmentation; MVGRL \cite{Hassani:2020un} proposes to learn both node- and graph-level representations by performing node diffusion and contrasting node representations to augmented graph summary representations.
Moreover, GCC \cite{Qiu:2020gq} proposes a pretraining framework based on CL. It proposes to construct multiple graph views by sampling subgraphs based on random walks and then learn model weights with several feature engineering schemes.
However, these methods do not explicitly consider adaptive graph augmentation at both structural and attribute levels, leading to suboptimal performance. Unlike these work, the adaptive data augmentation at both topology and attribute levels used in our GCA is able to preserve important patterns underneath the graph through stochastic perturbation.

\paragraph{Comparisons with related graph CL methods.}
In summary, we provide a brief comparison between the proposed GCA and other state-of-the-art graph contrastive representation learning methods, including DGI \cite{Velickovic:2019tu}, GMI \cite{Peng:2020gw}, and MVGRL \cite{Hassani:2020un} in Table \ref{tab:comparison}, where the last two columns denote data augmentation strategies at topology and attribute levels respectively.
It is seen that the proposed GCA method simplifies previous node--global contrastive scheme by defining contrastive objective at the node level. Most importantly, GCA is the only one that proposes adaptive data augmentation on both topology and attribute levels.

\begin{table}
	\centering
	\caption{Comparison with related work.}
	\begin{tabular}{cccc}
	\toprule
    Method & \makecell{Contrastive\\objective} & Topology & Attribute \\
    \midrule
    DGI   & Node--global & Uniform  & ---   \\
    GMI   & Node--node   & ---      & ---   \\
    MVGRL & Node--global & Uniform  & ---   \\
    \textbf{\themodel} & \textbf{Node--node} & \textbf{Adaptive} & \textbf{Adaptive} \\
    \bottomrule
    \end{tabular}
  \label{tab:comparison}
\end{table}

\makeatletter
\def\mathcenterto#1#2{\mathclap{\phantom{#1}\mathclap{#2}}\phantom{#1}}
\let\old@widetilde\widetilde
\def\widetildeto#1#2{\mathcenterto{#2}{\old@widetilde{\mathcenterto{#1}{#2\,}}}}
\let\old@widehat\widehat
\def\widehatto#1#2{\mathcenterto{#2}{\old@widehat{\mathcenterto{#1}{#2\,}}}}
\makeatother
\def\widetilde{\widetildeto{X}}

\section{The Proposed Method}
\label{sec:method}

In the following section, we present \themodel in detail, starting with the overall contrastive learning framework, followed by the proposed adaptive graph augmentation schemes. Finally, we provide theoretical justification behind our method.

\subsection{Preliminaries}
Let \(\mathcal{G} = (\mathcal{V}, \mathcal{E})\) denote a graph, where \(\mathcal{V} = \{ v_1, v_2, \cdots, v_N\}\), \(\mathcal{E} \subseteq \mathcal V \times \mathcal V\) represent the node set and the edge set respectively. We denote the feature matrix and the adjacency matrix as \(\bm{X} \in \mathbb{R}^{N \times F}\) and \(\bm{A} \in \{0,1\}^{N \times N}\), where \(\bm{x}_i \in \mathbb{R}^{F}\) is the feature of \(v_i\), and \(\bm{A}_{ij} = 1\) iff \((v_i, v_j) \in \mathcal{E}\).
There is no given class information of nodes in \(\mathcal{G}\) during training in the unsupervised setting.
Our objective is to learn a GNN encoder \(f(\bm{X}, \bm{A}) \in \mathbb{R}^{N \times F^\prime}\) receiving the graph features and structure as input, that produces node embeddings in low dimensionality, i.e. \(F^\prime \ll F\). We denote \(\bm{H} = f(\bm{X}, \bm{A})\) as the learned representations of nodes, where \(\bm{h}_i\) is the embedding of node \(v_i\). These representations can be used in downstream tasks, such as node classification and community detection.

\subsection{The Contrastive Learning Framework} 

The proposed \themodel framework follows the common graph CL paradigm where the model seeks to maximize the agreement of representations between different views \cite{Zhu:2020vf,Hassani:2020un}.
To be specific, we first generate two graph views by performing stochastic graph augmentation on the input. Then, we employ a contrastive objective that enforces the encoded embeddings of each node in the two different views to agree with each other and can be discriminated from embeddings of other nodes.

In our \themodel model, at each iteration, we sample two stochastic augmentation function \(t \sim \mathcal{T}\) and \(t' \sim \mathcal{T}\), where \(\mathcal{T}\) is the set of all possible augmentation functions. Then, we generate two graph views, denoted as \(\widetilde{\mathcal G}_1 = t(\widetilde{\mathcal G})\) and \(\widetilde{\mathcal G}_2 = t'(\widetilde{\mathcal G})\), and denote node embeddings in the two generated views as \(\bm{U} = f(\widetilde{\bm{X}}_1, \widetilde{\bm A}_1)\) and \(\bm{V} = f(\widetilde{\bm{X}}_2, \widetilde{\bm A}_2)\), where \(\widetilde{\bm{X}}_{\ast}\) and \(\widetilde{\bm{A}}_{\ast}\) are the feature matrices and adjacent matrices of the views.

After that, we employ a contrastive objective, i.e. a discriminator, that distinguishes the embeddings of the same node in these two different views from other node embeddings.
For any node \(v_i\), its embedding generated in one view, \(\bm{u}_i\), is treated as the anchor, the embedding of it generated in the other view, \(\bm{v}_i\), forms the positive sample, and the other embeddings in the two views are naturally regarded as negative samples.
Mirroring the InfoNCE objective \cite{vandenOord:2018ut} in our multi-view graph CL setting, we define the pairwise objective for each positive pair \((\bm{u}_i, \bm{v}_i)\) as
\begin{equation}
\begin{split}
	\ell &(\bm{u}_i, \bm{v}_i) =\\
	&\log \frac {e^{\theta\left(\bm{u}_i, \bm{v}_{i} \right) / \tau}} {\underbrace{e^{\theta\left(\bm{u}_i, \bm{v}_{i} \right) / \tau}}_{\text{positive pair}} + \underbrace{\sum_{k \neq i} e^{\theta\left(\bm{u}_i, \bm{v}_{k} \right) / \tau}}_{\text{inter-view negative pairs}} + \underbrace{\sum_{k \neq i}e^{\theta\left(\bm{u}_i, \bm{u}_k \right) / \tau}}_{\text{intra-view negative pairs}}},\\
\end{split}
\label{eq:pairwise-loss}
\end{equation}
where \(\tau\) is a temperature parameter. We define the critic \(\theta(\bm{u}, \bm{v}) = s(g(\bm{u}), g(\bm{v}))\), where \(s(\cdot, \cdot)\) is the cosine similarity and \(g(\cdot)\) is a non-linear projection to enhance the expression power of the critic function \cite{Chen:2020wj,Tschannen:2020uo}. The projection function \(g\) in our method is implemented with a two-layer perceptron model.

Given a positive pair, we naturally define negative samples as all other nodes in the two views. Therefore, negative samples come from two sources, that are inter-view and intra-view nodes, corresponding to the second and the third term in the denominator in Eq. (\ref{eq:pairwise-loss}), respectively. Since two views are symmetric, the loss for another view is defined similarly for \(\ell(\bm{v}_i, \bm u_i)\). 
The overall objective to be maximized is then defined as the average over all positive pairs, formally given by
\begin{equation}
	\mathcal{J} = \frac{1}{2N} \sum_{i = 1}^{N} \left[\ell(\bm{u}_i, \bm{v}_i) + \ell(\bm{v}_i, \bm{u}_i)\right].
	\label{eq:overall-loss}
\end{equation}

To sum up, at each training epoch, \themodel first draws two data augmentation functions \(t\) and \(t'\), and then generates two graph views \(\widetilde{\mathcal{G}}_1 = t(\mathcal{G})\) and \(\widetilde{\mathcal{G}}_2 = t'(\mathcal{G})\) of graph \(\mathcal{G}\) accordingly. Then, we obtain node representations \(\bm{U}\) and \(\bm{V}\) of \(\widetilde{\mathcal{G}}_1\) and \(\widetilde{\mathcal{G}}_2\) using a GNN encoder \(f\). Finally, the parameters are updated by maximizing the objective in Eq. (\ref{eq:overall-loss}). The training algorithm is summarized in Algorithm \ref{algo:training}.

\begin{algorithm}[ht]
	\DontPrintSemicolon\SetNoFillComment
	\caption{The \themodel training algorithm}
	\label{algo:training}
	\For {\(epoch \gets 1, 2, \cdots\)} {
		Sample two stochastic augmentation functions \(t \sim \mathcal{T}\) and \(t' \sim \mathcal{T}\)\;
		Generate two graph views \(\widetilde{\mathcal{G}}_1 = t(\mathcal{G})\) and \(\widetilde{\mathcal{G}}_2 = t'(\mathcal{G})\) by performing corruption on \(\mathcal{G}\)\;
		Obtain node embeddings \(\bm{U}\) of \(\widetilde{\mathcal{G}}_1\) using the encoder \(f\)\;
		Obtain node embeddings \(\bm{V}\) of \(\widetilde{\mathcal{G}}_2\) using the encoder \(f\)\;
		Compute the contrastive objective \(\mathcal{J}\) with Eq. (\ref{eq:overall-loss})\;
		Update parameters by applying stochastic gradient ascent to maximize \(\mathcal{J}\)\;
	}
\end{algorithm}

\subsection{Adaptive Graph Augmentation}

In essence, CL methods that maximize agreement between views seek to learn representations that are \emph{invariant} to perturbation introduced by the augmentation schemes \cite{Xiao:2020vt}.
In the \themodel model, we propose to design augmentation schemes that tend to keep important structures and attributes unchanged, while perturbing possibly unimportant links and features. Specifically, we corrupt the input graph by randomly removing edges and masking node features in the graph, and the removing or masking probabilities are skewed for unimportant edges or features, that is, higher for unimportant edges or features, and lower for important ones.
From an amortized perspective, we emphasize important structures and attributes over randomly corrupted views, which will guide the model to preserve fundamental topological and semantic graph patterns.

\subsubsection{Topology-level augmentation.}

For topology-level augmentation, we consider a direct way for corrupting input graphs where we randomly remove edges in the graph \cite{Zhu:2020vf}. Formally, we sample a modified subset \(\widetilde{\mathcal{E}}\) from the original \(\mathcal E\) with probability
\begin{equation}
	P \{ (u, v) \in \widetilde{\mathcal E} \} = 1 - p^e_{uv},
\end{equation}
where \((u, v) \in \mathcal{E}\) and \(p^e_{uv}\) is the probability of removing \((u, v)\). \(\widetilde {\mathcal{E}}\) is then used as the edge set in the generated view. \(p^e_{uv}\) should reflect the importance of the edge \((u, v)\) such that the augmentation function are more likely to corrupt unimportant edges while keep important connective structures intact in augmented views.

In network science, node centrality is a widely-used measure that quantifies the influence of nodes in the graph \cite{Newman:2018aw}. 
We define edge centrality \(w^e_{uv}\) for edge \((u, v)\) to measure its influence based on centrality of two connected nodes. Given a node centrality measure \(\varphi_c(\cdot) : \mathcal V \rightarrow \mathbb R^+\), we define edge centrality as the average of two adjacent nodes' centrality scores, i.e. \(w^e_{uv} = (\varphi_c(u) + \varphi_c(v))/2\), and on directed graph, we simply use the centrality of the tail node, i.e. \(w^e_{uv} = \varphi_c(v)\), since the importance of edges is generally characterized by nodes they are pointing to \cite{Newman:2018aw}.

Next, we calculate the probability of each edge based on its centrality value. Since node centrality values like degrees may vary across orders of magnitude \cite{Newman:2018aw}, we first set \(s^e_{uv} = \log w^e_{uv}\) to alleviate the impact of nodes with heavily dense connections. The probabilities can then be obtained after a normalization step that transform the values into probabilities, which is defined as
\begin{equation}
  p^e_{uv} = \min \left( \frac {s^e_{\max} - s^e_{uv}} {s^e_{\max} - \mu^e_s} \cdot p_e , p_\tau \right),
\end{equation}
where \(p_e\) is a hyperparameter that controls the overall probability of removing edges, \(s^e_{\max}\) and \(\mu^e_s\) is the maximum and average of \(s^e_{uv}\), and \(p_\tau < 1\) is a cut-off probability, used to truncate the probabilities since extremely high removal probabilities will lead to overly corrupted graph structures.

For the choice of the node centrality function, we use the following three centrality measures, including degree centrality, eigenvector centrality, and PageRank centrality due to their simplicity and effectiveness. 

\paragraph{Degree centrality.}
Node degree itself can be a centrality measure \cite{Newman:2018aw}.
On directed networks, we use in-degrees since the influence of a node in directed graphs are mostly bestowed by nodes pointing at it \cite{Newman:2018aw}. Despite that the node degree is one of the simplest centrality measures, it is quite effective and illuminating. For example, in citation networks where nodes represent papers and edges represent citation relationships, nodes with the highest degrees are likely to correspond to influential papers.

\paragraph{Eigenvector centrality.}
The eigenvector centrality \cite{Newman:2018aw,Bonacich:1987up} of a node is calculated as its eigenvector corresponding to the largest eigenvalue of the adjacency matrix.
Unlike degree centrality, which assumes that all neighbors contribute equally to the importance of the node, eigenvector centrality also takes the importance of neighboring nodes into consideration. By definition, the eigenvector centrality of each node is proportional to the sum of centralities of its neighbors, nodes that are either connected to many neighbors or connected to influential nodes will have high eigenvector centrality values. On directed graphs, we use the right eigenvector to compute the centrality, which corresponds to incoming edges. Note that since only the leading eigenvector is needed, the computational burden for calculating the eigenvector centrality is negligible.

\paragraph{PageRank centrality.}
The PageRank centrality \cite{Newman:2018aw,Page:1999wg} is defined as the PageRank weights computed by the PageRank algorithm. The algorithm propagates influence along directed edges, and nodes gathered the most influence are regarded as important nodes. Formally, the centrality values are defined by
\begin{equation}
  \bm\sigma = \alpha \bm{A} \bm{D}^{-1} \bm\sigma + \bm{1},
\end{equation}
where \(\sigma \in \mathbb{R}^N\) is the vector of PageRank centrality scores for each node and \(\alpha\) is a damping factor that prevents sinks in the graph from absorbing all ranks from other nodes connected to the sinks. We set \(\alpha=0.85\) as suggested in \citet{Page:1999wg}.
For undirected graphs, we execute PageRank on transformed directed graphs, where each undirected edge is converted to two directed edges.

\begin{figure}[b]
	\centering
	\subfloat[Degree]{
		\includegraphics[width=0.31\linewidth,frame]{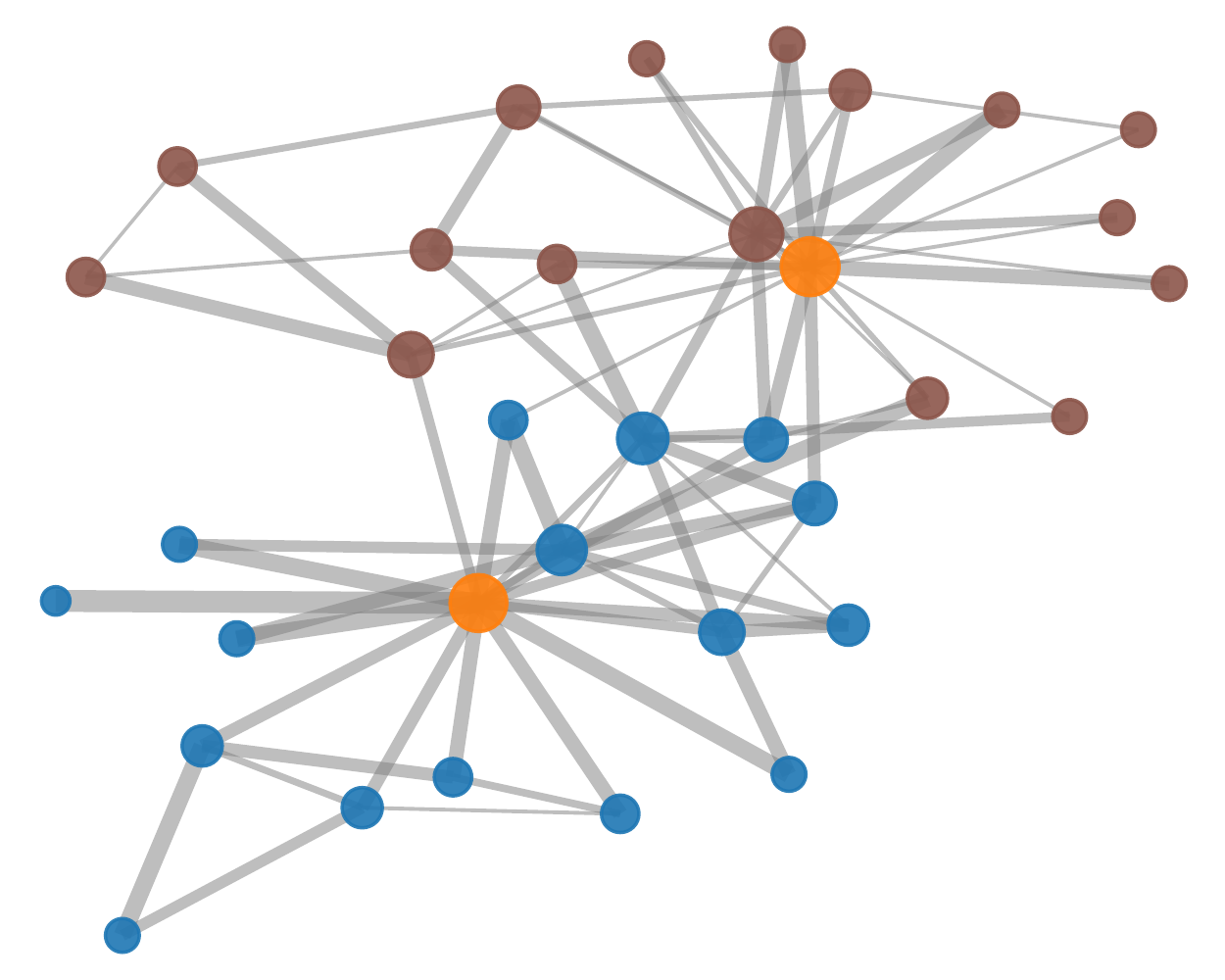}
	}
	\subfloat[Eigenvector]{
		\includegraphics[width=0.31\linewidth,frame]{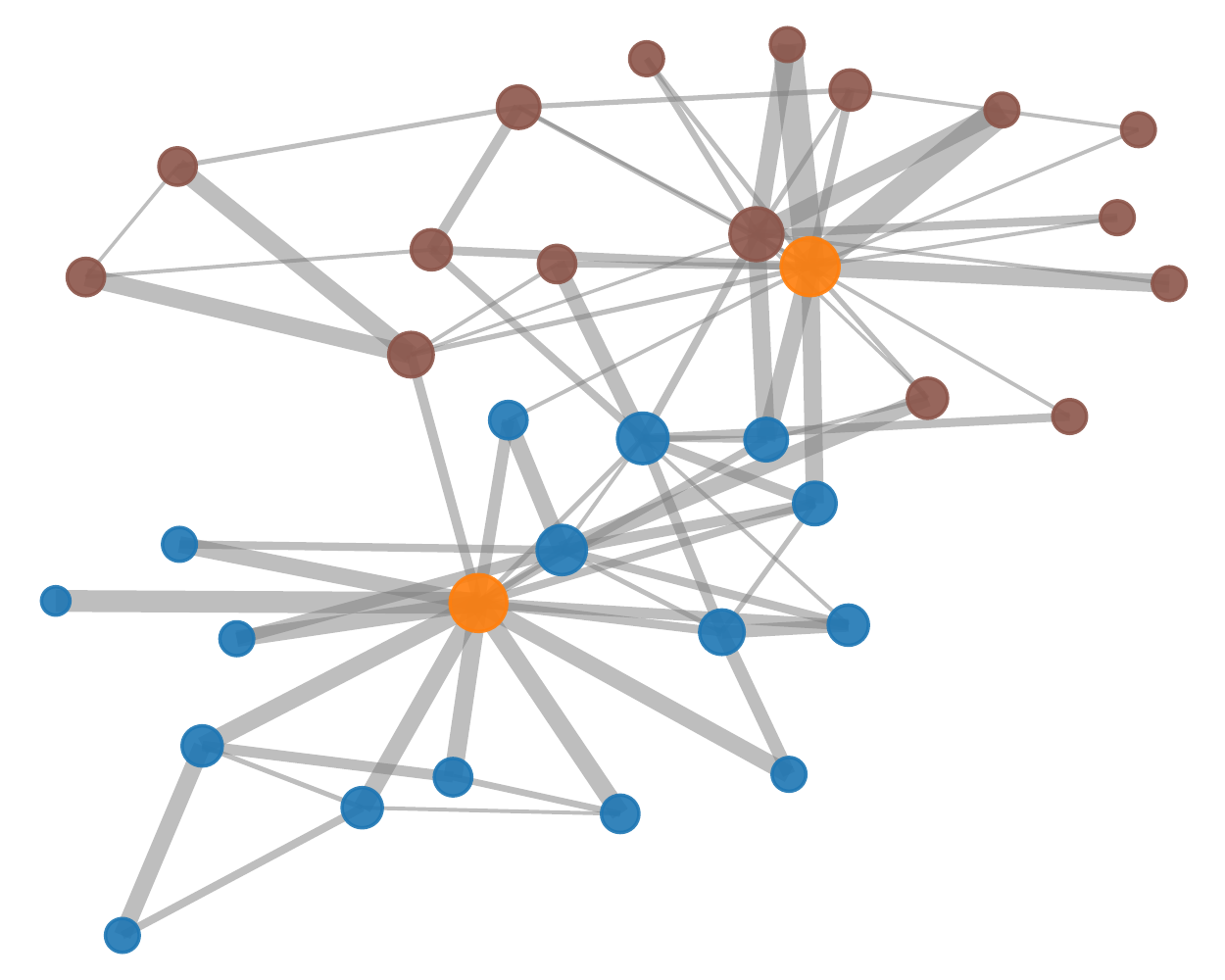}
	}
	\subfloat[PageRank]{
		\includegraphics[width=0.31\linewidth,frame]{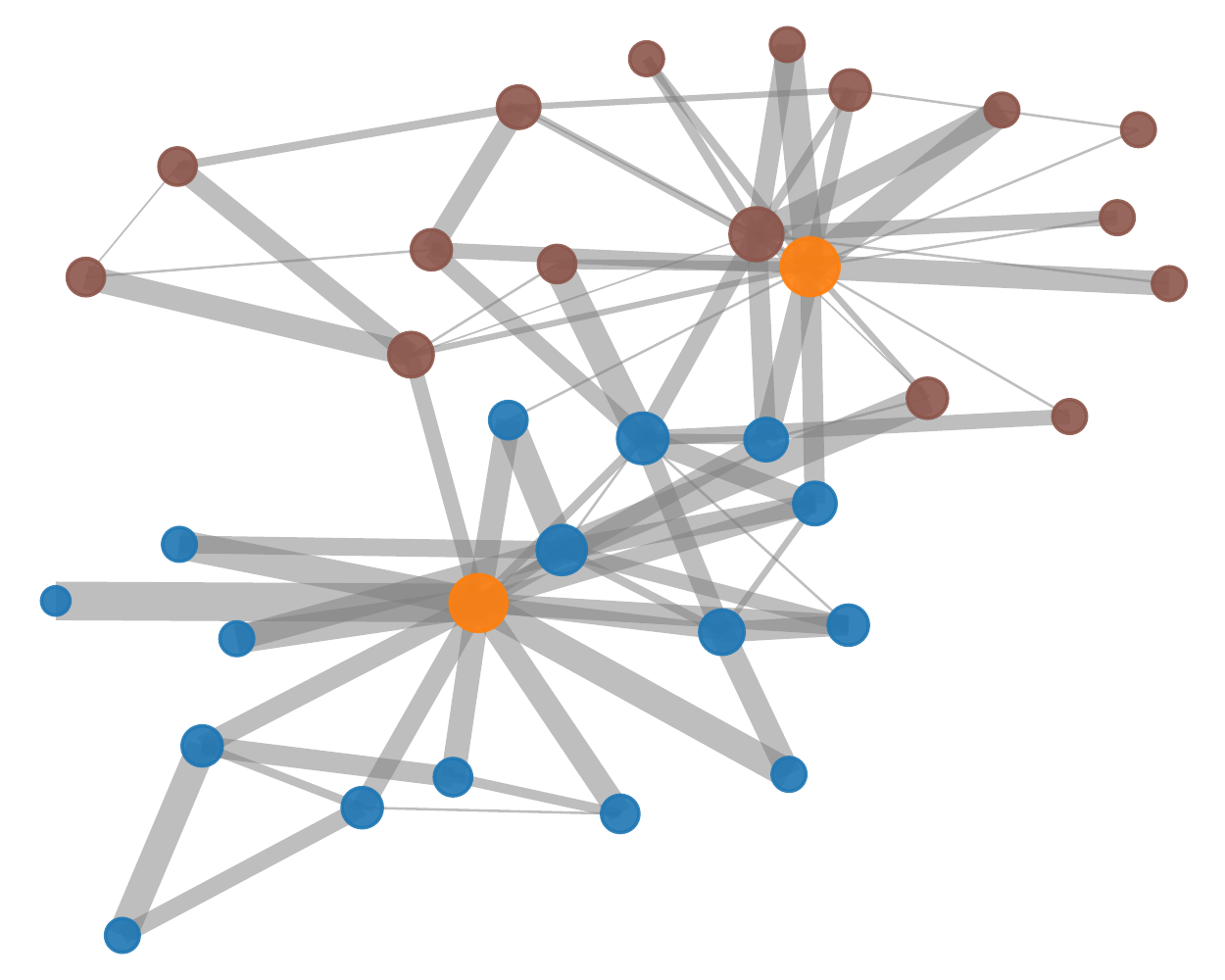}
	}
	\caption{Visualization of edge centrality computed by three schemes in the Karate club dataset, where centrality values are shown in terms of the thickness of edges. Node colors indicate two classes inside the network; two coaches are in orange.}
	\label{fig:visualization-edge-weights}
\end{figure}

To gain an intuition of these proposed adaptive structural augmentation schemes, we calculate edge centrality scores of the famous Karate club dataset \cite{Zachary:1977fs}, containing two groups of students leading by two coaches respectively. The edge centrality values calculated by different schemes are visualized in Figure \ref{fig:visualization-edge-weights}.
As can be seen in the figure, though the three schemes exhibit subtle differences, all of the augmentation schemes tend to emphasize edges that connect the two coaches (in orange) inside the two groups and put less attention to links between peripheral nodes across groups. This verifies that the proposed node-centrality-based adaptive topology augmentation scheme can recognize fundamental structures of the graph. 

\subsubsection{Node-attribute-level augmentation.}

On the node attribute level, similar to the salt-and-pepper noise in digital image processing \cite{Gonzalez:2018dp}, we add noise to node attributes via randomly masking a fraction of dimensions with zeros in node features. Formally, we first sample a random vector \(\widetilde{\bm{m}} \in \{ 0, 1 \}^F\) where each dimension of it independently is drawn from a Bernoulli distribution independently, i.e., \(\widetilde{m}_i \sim \text{Bern}(1 - p^f_i), \forall i\). Then, the generated node features \(\widetilde {\bm X}\) is computed by
\begin{equation}
	\widetilde {\bm{X}} = [ \bm{x}_1 \circ \widetilde{\bm{m}}; \bm{x}_2 \circ \widetilde{\bm{m}}; \cdots; \bm{x}_N \circ \widetilde{\bm{m}} ]^\top.
\end{equation}
Here \([\cdot ; \cdot ]\) is the concatenation operator, and \(\circ\) is the element-wise multiplication.

Similar to topology-level augmentation, the probability \(p^f_i\) should reflect the importance of the \(i\)-th dimension of node features. We assume that feature dimensions frequently appearing in influential nodes should be important, and define the weights of feature dimensions as follows. For sparse one-hot nodes features, i.e. \(x_{ui} \in \{ 0, 1 \}\) for any node \(u\) and feature dimension \(i\), we calculate the weight of dimension \(i\) as
\begin{equation}
  w^f_i = \sum_{u \in \mathcal V} x_{ui} \cdot \varphi_c (u),
\end{equation}
where \(\varphi_c (\cdot)\) is a node centrality measure that is used to quantify node importance. The first term \(x_{ui} \in \{ 0, 1 \}\) indicates the occurrence of dimension \(i\) in node \(u\), and the second term \(\varphi_i(u)\) measures the node importance of each occurrence.
To provide some intuition behind the above definition, consider a citation network where each feature dimension corresponds to a keyword. Then, keywords that frequently appear in a highly influential paper should be considered informative and important.

For dense, continuous node features \(\bm{x}_u\) of node \(u\), where \(x_{ui}\) denotes feature value at dimension \(i\), we cannot directly count the occurrence of each one-hot encoded value. Then, we turn to measure the magnitude of the feature value at dimension \(i\) of node \(u\) by its absolute value \(| x_{ui} |\). Formally, we calculate the weights by
\begin{equation}
	w^f_i = \sum_{u \in \mathcal V} |x_{ui}| \cdot \varphi_c (u).
\end{equation}

Similar to topology augmentation, we perform normalization on the weights to obtain the probability representing feature importance. Formally,
\begin{equation}
	p_i^f = \min \left( \frac {s^f_{\max} - s^f_{i}} {s^f_{\max} - \mu_s^f} \cdot p_f, p_\tau \right),
\end{equation}
where \(s^f_i = \log w^f_i\), \(s^f_{\max}\) and \(\mu^f_s\) is the maximum and the average value of \(s^f_i\) respectively, and \(p_f\) is a hyperparameter that controls the overall magnitude of feature augmentation.

Finally, we generate two corrupted graph views \(\widetilde {\mathcal{G}}_1, \widetilde {\mathcal{G}}_2\) by jointly performing topology- and node-attribute-level augmentation. In \themodel, the probability \(p_e\) and \(p_f\) is different for generating the two views to provide a diverse context for contrastive learning, where the probabilities for the first and the second view are denoted by \(p_{e,1}, p_{f,1}\) and \(p_{e,2}, p_{f,2}\) respectively.

In this paper, we propose and evaluate three model variants, denoted as \themodel-DE, \themodel-EV, and \themodel-PR. The three variants employ degree, eigenvector, and PageRank centrality measures respectively.
Note that all centrality and weight measures are only dependent on the topology and node attributes of the original graph. Therefore, they only need to be computed once and do not bring much computational burden.

\subsection{Theoretical Justification}

In this section, we provide theoretical justification behind our model from two perspectives, i.e. MI maximization and the triplet loss. Detailed proofs can be found in Appendix \ref{appendix:proofs}.

\paragraph{Connections to MI maximization.}
Firstly, we reveal the connections between our loss and MI maximization between node features and the embeddings in the two views. The InfoMax principle has been widely applied in representation learning literature \cite{Tian:2019vw,Bachman:2019wp,Poole:2019vk,Tschannen:2020uo}. MI quantifies the amount of information obtained about one random variable by observing the other random variable.
\begin{theorem}
\label{thm:objective-InfoMax}
Let \(\bm{X}_i = \{ \bm{x}_k \}_{k \in \mathcal{N}(i)}\) be the neighborhood of node \(v_i\) that collectively maps to its output embedding, where \(\mathcal{N}(i)\) denotes the set of neighbors of node \(v_i\) specified by GNN architectures, and \(\bm{X}\) be the corresponding random variable with a uniform distribution \(p(\bm{X}_i) = \nicefrac{1}{N}\). Given two random variables \(\bm{U, V} \in \mathbb{R}^{F'}\) being the embedding in the two views, with their joint distribution denoted as \(p(\bm{U}, \bm{V})\), our objective \(\mathcal{J}\) is a lower bound of MI between encoder input \(\bm{X}\) and node representations in two graph views \(\bm{U, V}\). Formally,
\begin{equation}
	\mathcal{J} \leq I(\bm{X}; \bm{U}, \bm{V}).
\end{equation}
\end{theorem}
\begin{proof}[Proof sketch]
We first observe that our objective \(\mathcal{J}\) is a lower bound of the InfoNCE objective \cite{vandenOord:2018ut,Poole:2019vk}, defined by \(I_\text{NCE}(\bm U; \bm V) \triangleq \mathbb{E}_{\prod_i {p(\bm u_i, \bm v_i)}} \left[ \frac{1}{N} \sum_{i=1}^N \log \frac{e^{\theta(\bm{u}_i, \bm{v}_i)}}{\frac{1}{N}\sum_{j = 1}^{N} e^{\theta(\bm{u}_i, \bm{v}_j)}} \right] \). Since the InfoNCE estimator is a lower bound of the true MI, the theorem directly follows from the application of data processing inequality \cite{Cover:2006ei}, which states that \(I(\bm U; \bm V) \leq I(\bm X; \bm U, \bm V)\).
\end{proof}
\begin{remark}
Theorem \ref{thm:objective-InfoMax} reveals that maximizing \(\mathcal{J}\) is equivalent to explicitly maximizing a lower bound of the MI \(I(\bm X; \bm U, \bm V)\) between input node features and learned node representations.
Recent work further provides empirical evidence that optimizing a stricter bound of MI may not lead to better downstream performance on visual representation learning \cite{Tschannen:2020uo,Tian:2020vw}, which further highlights the importance of the design of data augmentation strategies. 

When optimizing \(I(\bm U; \bm V)\), a lower bound of \(I(\bm X; \bm U, \bm V)\), we encourage the model to encode shared information between the two views. From the amortized perspective, corrupted views will follow a skewed distribution where important link structures and features are emphasized. By contrasting the two views, the model is enforced to encode the emphasized information into representations, which improves embedding quality.

However, as the objective is not defined specifically on negative samples generated by the augmentation function, it remains challenging to derive the relationship between specific augmentation functions and the lower bound. We shall leave it for future work.
\end{remark}

\paragraph{Connections to the triplet loss.}
Alternatively, we may also view the optimization problem in Eq. (\ref{eq:overall-loss}) as a classical triplet loss, commonly used in deep metric learning.
\begin{theorem}
\label{thm:objective-triplet-loss}
When the projection function \(g\) is the identity function and we measure embedding similarity by simply taking the inner product, i.e. \(s(\bm{u}, \bm{v}) = \bm{u}^\top \bm{v}\), and further assuming that positive pairs are far more aligned than negative pairs, i.e. \(\bm{u}_i^\top \bm{v}_k \ll \bm{u}_i^\top \bm{v}_i\) and \(\bm{u}_i^\top \bm{u}_k \ll \bm{u}_i^\top \bm{v}_i\), minimizing the pairwise objective \(\ell(\bm{u}_i, \bm{v}_i)\) coincides with maximizing the triplet loss, as given in the sequel
\begin{equation}
\begin{split}
	& - \ell (\bm u_i, \bm v_i) \propto \\
	& 4N \tau + \sum_{j \neq i}\left( \| {\bm u_i} - {\bm v_i} \|^2 - \| {\bm u_i} - {\bm v_j} \|^2 + \| {\bm u_i} - {\bm v_i} \|^2 - \| {\bm u_i} - {\bm u_j} \|^2\right).
\end{split}
\end{equation}
\end{theorem}
\begin{remark}
Theorem \ref{thm:objective-triplet-loss} draws connections between the objective and the classical triplet loss.
In other words, we may regard the problem in Eq. (\ref{eq:overall-loss}) as learning graph convolutional encoders to encourage positive samples being further away from negative samples in the embedding space.
Moreover, by viewing the objective from the metric learning perspective, we highlight the importance of appropriate data augmentation schemes, which is often neglected in previous InfoMax-based methods. Specifically, as the objective pulls together representation of each node in the two corrupted views, the model is enforced to encode information in the input graph that is insensitive to perturbation. Since the proposed adaptive augmentation schemes tend to keep important link structures and node attributes intact in the perturbation, the model is guided to encode essential structural and semantic information into the representation, which improves the quality of embeddings.
Last, the contrastive objective used in \themodel is cheap to optimize, since we do not have to generate negative samples explicitly and all computation can be performed in parallel. In contrast, the triplet loss is known to be computationally expensive \cite{Schroff:2015wo}.
\end{remark}

\section{Experiments}
\label{sec:experiments}

In this section, we conduct experiments to evaluate our model through answering the following questions.
\begin{itemize}
	\item \textbf{RQ1}. Does our proposed \themodel outperform existing baseline methods on node classification?
	\item \textbf{RQ2}. Do all proposed adaptive graph augmentation schemes benefit the learning of the proposed model? How does each graph augmentation scheme affect model performance?
	\item \textbf{RQ3}. Is the proposed model sensitive to hyperparameters? How do key hyperparameters impact the model performance?
\end{itemize}
We begin with a brief introduction of the experimental setup, and then we proceed to details of experimental results and their analysis.

\subsection{Experimental Setup}

\subsubsection{Datasets}

For comprehensive comparison, we use six widely-used datasets, including Wiki-CS, Amazon-Computers, Amazon-Photo, Coauthor-CS, and Coauthor-Physics, to study the performance of transductive node classification. The datasets are collected from real-world networks from different domains; their detailed statistics is summarized in Table \ref{tab:dataset-statistics}.
\begin{itemize}
	\item \textbf{Wiki-CS} \cite{Mernyei:2020wh} is a reference network constructed based on Wikipedia. The nodes correspond to articles about computer science and edges are hyperlinks between the articles. Nodes are labeled with ten classes each representing a branch of the field. Node features are calculated as the average of pretrained GloVe \cite{Pennington:2014kh} word embeddings of words in each article. 
	\item \textbf{Amazon-Computers} and \textbf{Amazon-Photo} \cite{Shchur:2018vv} are two networks of co-purchase relationships constructed from Amazon, where nodes are goods and two goods are connected when they are frequently bought together. Each node has a sparse bag-of-words feature encoding product reviews and is labeled with its category. 
	\item \textbf{Coauthor-CS} and \textbf{Coauthor-Physics} \cite{Shchur:2018vv} are two academic networks, which contain co-authorship graphs based on the Microsoft Academic Graph from the KDD Cup 2016 challenge. In these graphs, nodes represent authors and edges indicate co-authorship relationships; that is, two nodes are connected if they have co-authored a paper. Each node has a sparse bag-of-words feature based on paper keywords of the author. The label of an author corresponds to their most active research field.
\end{itemize}
Among these datasets, Wiki-CS has dense numerical features, while the other four datasets only contain sparse one-hot features.
For the Wiki-CS dataset, we evaluate the models on the public splits shipped with the dataset \cite{Mernyei:2020wh}.
Regarding the other four datasets, since they have no public splits available, we instead randomly split the datasets, where 10\%, 10\%, and the rest 80\% of nodes are selected for the training, validation, and test set, respectively.

\begin{table}[t]
\begin{threeparttable}
	\small
	\centering
	\caption{Statistics of datasets used in experiments.}
	\label{tab:dataset-statistics}
	\begin{tabular}{cccccc}
	\toprule
	Dataset & \#Nodes & \#Edges & \#Features & \#Classes \\
	\midrule
	Wiki-CS\tnotex{fn:wikics} & 11,701 & 216,123 & 300 & 10 \\
	Amazon-Computers\tnotex{fn:amazon-computers} & 13,752 & 245,861 & 767 & 10 \\
	Amazon-Photo\tnotex{fn:amazon-photo} & 7,650 & 119,081 & 745 & 8 \\
	Coauthor-CS\tnotex{fn:coauthor-cs} & 18,333 & 81,894 & 6,805 & 15 \\
	Coauthor-Physics\tnotex{fn:coauthor-phy} & 34,493 & 247,962 & 8,415 & 5 \\
	\bottomrule
	\end{tabular}
	\begin{tablenotes}[flushleft]
	\scriptsize{
		\item[1] \label{fn:wikics} \url{https://github.com/pmernyei/wiki-cs-dataset/raw/master/dataset}
		\item[2] \label{fn:amazon-computers} \url{https://github.com/shchur/gnn-benchmark/raw/master/data/npz/amazon_electronics_computers.npz}
		\item[3] \label{fn:amazon-photo} \url{https://github.com/shchur/gnn-benchmark/raw/master/data/npz/amazon_electronics_photo.npz}
		\item[4] \label{fn:coauthor-cs} \url{https://github.com/shchur/gnn-benchmark/raw/master/data/npz/ms_academic_cs.npz}
		\item[5] \label{fn:coauthor-phy} \url{https://github.com/shchur/gnn-benchmark/raw/master/data/npz/ms_academic_phy.npz}
	}
	\end{tablenotes}
\end{threeparttable}
\addtocounter{footnote}{+5}

\end{table}

\subsubsection{Evaluation protocol.}
For every experiment, we follow the linear evaluation scheme as introduced in \citet{Velickovic:2019tu}, where each model is firstly trained in an unsupervised manner; then, the resulting embeddings are used to train and test a simple \(\ell_2\)-regularized logistic regression classifier.
We train the model for twenty runs for different data splits and report the averaged performance on each dataset for fair evaluation.
Moreover, we measure performance in terms of accuracy in these experiments.

\subsubsection{Baselines.}
We consider representative baseline methods belonging to the following two categories: (1) traditional methods including DeepWalk \cite{Perozzi:2014ib} and node2vec \cite{Grover:2016ex} and (2) deep learning methods including Graph Autoencoders (GAE, VGAE) \cite{Kipf:2016ul}, Deep Graph Infomax (DGI) \cite{Velickovic:2019tu}, Graphical Mutual Information Maximization (GMI) \cite{Peng:2020gw}, and Multi-View Graph Representation Learning (MVGRL) \cite{Hassani:2020un}.
Furthermore, we report the performance obtained using a logistic regression classifier on raw node features and DeepWalk with embeddings concatenated with input node features.
To directly compare our proposed method with supervised counterparts, we also report the performance of two representative models Graph Convolutional Networks (GCN) \cite{Kipf:2016tc} and Graph Attention Networks (GAT) \cite{Velickovic:2018we}, where they are trained in an end-to-end fashion.
For all baselines, we report their performance based on their official implementations.

\subsubsection{Implementation details.}
We employ a two-layer GCN \cite{Kipf:2016tc} as the encoder for all deep learning baselines due to its simplicity. The encoder architecture is formally given by
\begin{align}
	\GC_i (\bm{X}, \bm{A}) & = \sigma \left( \hat{\bm{D}}^{-\frac{1}{2}} \hat {\bm{A}} \hat{\bm{D}}^{-\frac{1}{2}} \bm{X} \bm{W}_i \right), \\
	f(\bm X, \bm A) & = \GC_2 ( \GC_1 ( \bm{X}, \bm{A} ), \bm{A} ).
\end{align}
where \(\hat{\bm{A}} = \bm{A} + \bm{I}\) is the adjacency matrix with self-loops, \(\hat {\bm D} = \sum_i \hat{\bm{A}}_i\) is the degree matrix, \(\sigma(\cdot)\) is a nonlinear activation function, e.g., \(\operatorname{ReLU}(\cdot) = \max(0, \cdot)\), and \(\bm{W}_i\) is a trainable weight matrix.
For experimental specifications, including details of the configurations of the optimizer and hyperparameter settings, we refer readers of interest to Appendix \ref{appendix:implementation}.

\begin{table*}[t]
	\centering
	\caption{Summary of performance on node classification in terms of accuracy in percentage with standard deviation. Available data for each method during the training phase is shown in the second column, where \(\bm{X}, \bm{A}, \bm{Y}\) correspond to node features, the adjacency matrix, and labels respectively. The highest performance of unsupervised models is highlighted in boldface; the highest performance of supervised models is underlined. OOM indicates Out-Of-Memory on a 32GB GPU.}
	\label{tab:node-classification}
	\begin{tabular}{ccccccc}
	\toprule
	Method & Training Data & Wiki-CS & Amazon-Computers & Amazon-Photo & Coauthor-CS & Coauthor-Physics  \\
	\midrule
	Raw features & \(\bm{X}\) & 71.98 ± 0.00 & 73.81 ± 0.00 & 78.53 ± 0.00 & 90.37 ± 0.00 & 93.58 ± 0.00 \\
	node2vec & \(\bm{A}\) & 71.79 ± 0.05 & 84.39 ± 0.08 & 89.67 ± 0.12 & 85.08 ± 0.03 & 91.19 ± 0.04 \\
	DeepWalk & \(\bm{A}\) & 74.35 ± 0.06 & 85.68 ± 0.06 & 89.44 ± 0.11 & 84.61 ± 0.22 & 91.77 ± 0.15 \\
	DeepWalk + features & \(\bm{X}, \bm{A}\) & 77.21 ± 0.03 & 86.28 ± 0.07 & 90.05 ± 0.08 & 87.70 ± 0.04 & 94.90 ± 0.09 \\
	\midrule
	GAE & \(\bm{X}, \bm{A}\) & 70.15 ± 0.01 & 85.27 ± 0.19 & 91.62 ± 0.13 & 90.01 ± 0.71 & 94.92 ± 0.07 \\
	VGAE & \(\bm{X}, \bm{A}\) & 75.63 ± 0.19 & 86.37 ± 0.21 & 92.20 ± 0.11 & 92.11 ± 0.09 & 94.52 ± 0.00 \\
	DGI & \(\bm{X}, \bm{A}\) & 75.35 ± 0.14 & 83.95 ± 0.47 & 91.61 ± 0.22 & 92.15 ± 0.63 & 94.51 ± 0.52 \\
	GMI & \(\bm{X}, \bm{A}\) & 74.85 ± 0.08 & 82.21 ± 0.31 & 90.68 ± 0.17 & OOM & OOM \\
	MVGRL & \(\bm{X}, \bm{A}\) & 77.52 ± 0.08 & 87.52 ± 0.11 & 91.74 ± 0.07 & 92.11 ± 0.12 & 95.33 ± 0.03 \\
	\rowcolor{lightgray!50} \textbf{\themodel-DE} & \(\bm{X}, \bm{A}\) & 78.30 ± 0.00 & \textbf{87.85 ± 0.31} & 92.49 ± 0.09 & \textbf{93.10 ± 0.01} & 95.68 ± 0.05 \\
	\rowcolor{lightgray!50} \textbf{\themodel-PR} & \(\bm{X}, \bm{A}\) & \textbf{78.35 ± 0.05} & 87.80 ± 0.23 & \textbf{92.53 ± 0.16} & 93.06 ± 0.03 & 95.72 ± 0.03 \\
	\rowcolor{lightgray!50} \textbf{\themodel-EV} & \(\bm{X}, \bm{A}\) & 78.23 ± 0.04 & 87.54 ± 0.49 & 92.24 ± 0.21 & 92.95 ± 0.13 & \textbf{95.73 ± 0.03} \\
	\specialrule{0.5pt}{0.5pt}{1pt}
	\midrule
	GCN & \(\bm{X}, \bm{A}, \bm{Y}\) & 77.19 ± 0.12 & 86.51 ± 0.54 & 92.42 ± 0.22 & \underline{93.03 ± 0.31} & \underline{95.65 ± 0.16} \\
	GAT & \(\bm{X}, \bm{A}, \bm{Y}\) & \underline{77.65 ± 0.11} & \underline{86.93 ± 0.29} & \underline{92.56 ± 0.35} & 92.31 ± 0.24 & 95.47 ± 0.15 \\
	\bottomrule
	\end{tabular}
\end{table*}

\begin{table*}[t]
	\centering
	\caption{Performance of model variants on node classification in terms of accuracy in percentage with standard deviation. We use the degree centrality in all variants. The highest performance is highlighted in boldface.}
	\label{tab:ablation-study}
	\begin{tabular}{cccccccc}
	\toprule
	Variant & Topology & Attribute & Wiki-CS & Amazon-Computers & Amazon-Photo & Coauthor-CS & Coauthor-Physics \\
	\midrule
	\themodel--T--A & Uniform  & Uniform &  78.19 ± 0.01 & 86.25 ± 0.25 & 92.15 ± 0.24 & 92.93 ± 0.01 & 95.26 ± 0.02 \\
	\themodel--T & Uniform  & Adaptive & 78.23 ± 0.02 & 86.72 ± 0.49 & 92.20 ± 0.26 & 93.07 ± 0.01 & 95.59 ± 0.04 \\
	\themodel--A & Adaptive & Uniform & 78.25 ± 0.02 & 87.66 ± 0.30 & 92.23 ± 0.20 & 93.02 ± 0.01 & 95.54 ± 0.02 \\
	\rowcolor{lightgray!50} \textbf{\themodel} & Adaptive & Adaptive &  \textbf{78.30 ± 0.01} & \textbf{87.85 ± 0.31} & \textbf{92.49 ± 0.09} & \textbf{93.10 ± 0.01} & \textbf{95.68 ± 0.05} \\
	\bottomrule
	\end{tabular}
\end{table*}

\subsection{Performance on Node Classification (RQ1)}

The empirical performance is summarized in Table \ref{tab:node-classification}. Overall, from the table, we can see that our proposed model shows strong performance across all five datasets.
\themodel consistently performs better than unsupervised baselines by considerable margins on both transductive tasks. The strong performance verifies the superiority of the proposed contrastive learning framework.
On the two Coauthor datasets, we note that existing baselines have already obtained high enough performance; our method \themodel still pushes that boundary forward. 
Moreover, we particularly note that \themodel is competitive with models \emph{trained with label supervision} on all five datasets.

We make other observations as follows.
Firstly, the performance of traditional contrastive learning methods like DeepWalk is inferior to the simple logistic regression classifier that only uses raw features on some datasets (Coauthor-CS and Coauthor-Physics), which suggests that these methods may be ineffective in utilizing node features.
Unlike traditional work, we see that GCN-based methods, e.g., GAE, are capable of incorporating node features when learning embeddings. However, we note that on certain datasets (Wiki-CS), their performance is still worse than DeepWalk + feature, which we believe can be attributed to their na\"ive method of selecting negative samples that simply chooses contrastive pairs based on edges. This fact further demonstrates the important role of selecting negative samples based on augmented graph views in contrastive representation learning.
Moreover, compared to existing baselines DGI, GMI, and MVGRL, our proposed method performs strong, adaptive data augmentation in constructing negative samples, leading to better performance. Note that, although MVGRL employs diffusion to incorporate global information into augmented views, it still fails to consider the impacts of different edges adaptively on input graphs. The superior performance of \themodel verifies that our proposed adaptive data augmentation scheme is able to help improve embedding quality by preserving important patterns during perturbation. 

Secondly, we observe that all three variants with different node centrality measures of \themodel outperform existing contrastive baselines on all datasets.
We also notice that \themodel-DE and \themodel-PR with the degree and PageRank centrality respectively are two strong variants that achieve the best or competitive performance on all datasets.
Please kindly note that the result indicates that our model is not limited to specific choices of centrality measures and verifies the effectiveness and generality of our proposed framework.

In summary, the superior performance of \themodel compared to existing state-of-the-art methods verifies the effectiveness of our proposed \themodel framework that performs data augmentation adaptive to the graph structure and attributes.

\subsection{Ablation Studies (RQ2)}

In this section, we substitute the proposed topology and attribute level augmentation with their uniform counterparts to study the impact of each component of \themodel.
\themodel--T--A denotes the model with uniform topology and node attribute augmentation schemes, where the probabilities of dropping edge and masking features are set to the same for all nodes. The variants \themodel--T and \themodel--A are defined similarly except that we substitute the topology and the node attribute augmentation scheme with uniform sampling in the two models respectively. Degree centrality is used in all the variants for fair comparison. Please kindly note that the downgraded \themodel--T--A fallbacks to our preliminary work GRACE \cite{Zhu:2020vf}.

The results are presented in Table \ref{tab:ablation-study}, where we can see that both topology-level and node-attribute-level adaptive augmentation scheme improve model performance consistently on all datasets.
In addition, the combination of adaptive augmentation schemes on the two levels further benefits the performance. On the Amazon-Computers dataset, our proposed \themodel gains 1.5\% absolute improvement compared to the base model with no adaptive augmentation enabled. The results verify the effectiveness of our adaptive augmentation schemes on both topology and node attribute levels.

\subsection{Sensitivity Analysis (RQ3)}

\begin{figure}[t]
	\centering
	\includegraphics[width=\linewidth]{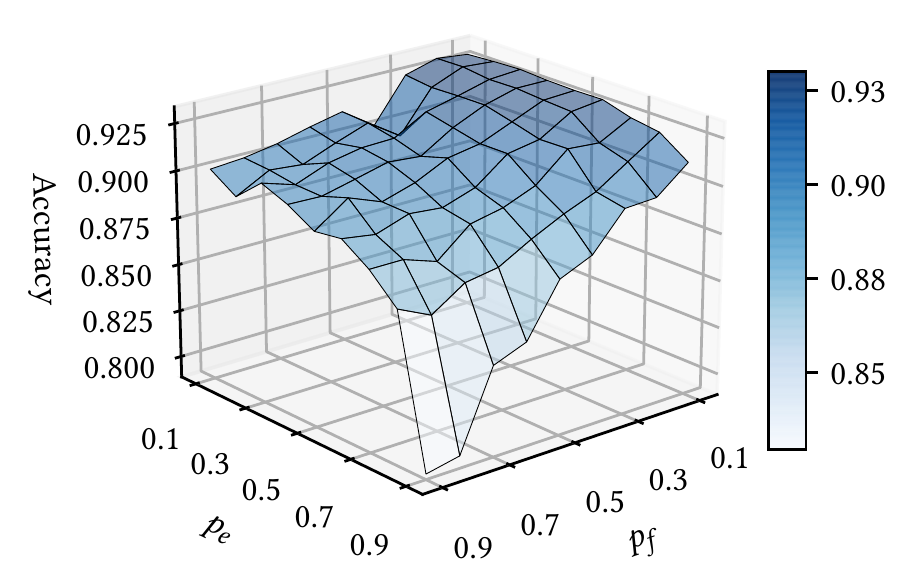}
	\caption{The performance of \themodel with varied different hyperparameters on the Amazon-Photo dataset in terms of node classification accuracy.}
	\label{fig:sensitivity-analysis}
\end{figure}

In this section, we perform sensitivity analysis on critical hyperparameters in \themodel, namely four probabilities \(p_{e,1},p_{f,1},p_{e,2}\), and \(p_{f,2}\) that determine the generation of graph views to show the stability of the model under perturbation of these hyperparameters.
We conduct transductive node classification by varying these parameters from 0.1 to 0.9. For sake of visualization brevity, we set \(p_e = p_{e,1} = p_{e,2}\) and \(p_f = p_{f,1} = p_{f,2}\) to control the magnitude of the proposed topology and node attribute level augmentation. We only change these four parameters in the sensitivity analysis, and other parameters remain the same as previously described.

The results on the Amazon-Photo dataset are shown in Figure \ref{fig:sensitivity-analysis}.
From the figure, it can be observed that the performance of node classification in terms of accuracy is relatively stable when the parameters are not too large, as shown in the plateau in the figure. We thus conclude that, overall, our model is insensitive to these probabilities, demonstrating the robustness to hyperparameter perturbation.
If the probability is set too large (e.g., \(> 0.5\)), the original graph will be heavily undermined. For example, when \(p_e = 0.9\), almost every existing edge has been removed, leading to isolated nodes in the generated graph views. Under such circumstances, the GNN is hard to learn useful information from node neighborhoods. Therefore, the learned node embeddings in the two graph views are not distinctive enough, which will result in the difficulty of optimizing the contrastive objective.

\section{Conclusion}
\label{sec:conclusion}

In this paper, we have developed a novel graph contrastive representation learning framework with adaptive augmentation.
Our model learns representation by maximizing the agreement of node embeddings between views that are generated by adaptive graph augmentation. The proposed adaptive augmentation scheme first identifies important edges and feature dimensions via network centrality measures. Then, on the topology level, we randomly remove edges by assigning large probabilities on unimportant edges to enforce the model to recognize network connectivity patterns. On the node attribute level, we corrupt attributes by adding more noise to unimportant feature dimensions to emphasize the underlying semantic information.
We have conducted comprehensive experiments using various real-world datasets. Experimental results demonstrate that our proposed \themodel method consistently outperforms existing state-of-the-art methods and even surpasses several supervised counterparts.

\begin{acks}
The authors would like to thank anonymous reviewers for their insightful comments.
The authors would also like to thank Mr. Tao Sun and Mr. Sirui Lu for their valuable discussion.
This work is jointly supported by National Natural Science Foundation of China (U19B2038, 61772528) and Beijing National Natural Science Foundation (4182066).
\end{acks}

\section*{Discussions on Broader Impact}
This paper presents a novel graph contrastive learning framework, and we believe it would be beneficial to the graph machine learning community both theoretically and practically.
Our proposed self-supervised graph representation learning techniques help alleviate the label scarcity issue when deploying machine learning applications in real-world, which saves a lot of efforts on human annotating. For example, our \themodel framework can be plugged into existing recommender systems and produces high-quality embeddings for users and items to resolve the cold start problem.
Note that our work mainly serves as a plug-in for existing machine learning models, it does not bring new ethical concerns. However, the \themodel model may still give biased outputs (e.g., gender bias, ethnicity bias), as the provided data itself may be strongly biased during the processes of the data collection, graph construction, etc.

\appendix
\section{Implementation Details}
\label{appendix:implementation}

\begin{table*}[t]
	\centering
	\caption{Hypeparameter specifications.}
    \begin{tabular}{ccccccccccc}
	\toprule
	Dataset & \(p_{e,1}\) & \(p_{e,2}\) & \(p_{f,1}\) & \(p_{f,2}\) & \(p_\tau\) & \(\tau\) & \makecell{Learning\\rate} & \makecell{Training\\epochs} & \makecell{Hidden\\dimension} & \makecell{Activation\\function} \\
	\midrule
	Wiki-CS  & 0.2   & 0.4   & 0.1 & 0.1  & 0.7 & 0.6 & 0.01 & 3,000   & 256  & PReLU \\
	Amazon-Computers & 0.5   & 0.5   & 0.2   & 0.1  & 0.7 & 0.1 & 0.01 & 1,500   & 128  & PReLU \\
	Amazon-Photo & 0.3   & 0.5   & 0.1   & 0.1  & 0.7 & 0.3 & 0.1 & 2,000  & 256   & ReLU \\
	Coauthor-CS  & 0.3   & 0.2   & 0.3   & 0.4  & 0.7 & 0.4 & 0.0005 & 1,000  & 256   & RReLU \\
	Coauthor-Physics  & 0.4   & 0.1   & 0.1   & 0.4  & 0.7 & 0.5 & 0.01 & 1,500  & 128   & RReLU \\
	\bottomrule
	\end{tabular}
	\label{tab:hyperparameters}
\end{table*}

\subsection{Computing Infrastructures}
\paragraph{Software infrastructures.}
All models are implemented using PyTorch Geometric 1.6.1 \cite{Fey:2019wv}, PyTorch 1.6.0 \cite{Paszke:2019vf}, and NetworkX 2.5 \cite{Hagberg:2008tk}. All datasets used throughout experiments are available in PyTorch Geometric libraries.
\paragraph{Hardware infrastructures.}
We conduct experiments on a computer server with four NVIDIA Tesla V100S GPUs (with 32GB memory each) and twelve Intel Xeon Silver 4214 CPUs.

\subsection{Hyperparameter Specifications}
All model parameters are initialized with Glorot initialization \cite{Glorot:2010uc}, and trained using Adam SGD optimizer \cite{Kingma:2015us} on all datasets.
The \(\ell_2\) weight decay factor is set to \(10^{-5}\) and the dropout rate \cite{Srivastava:2014cg} is set to zero on all datasets.
The probability parameters controlling the sampling process, \(p_{e,1}, p_{f,1}\) for the first view and \(p_{e,2}, p_{f,2}\) for the second view, are all selected between 0.0 and 0.4, since the original graph will be overly corrupted when the probability is set too large. Note that to generate different contexts for nodes in the two views, \(p_{e,1}\) and \(p_{e,2}\) should be distinct, and the same holds for \(p_{f,1}\) and \(p_{f,2}\).
All dataset-specific hyperparameter configurations are summarized in Table \ref{tab:hyperparameters}.

\section{Detailed Proofs}
\label{appendix:proofs}

\subsection{Proof of Theorem 1}

\begin{theorem}
\label{thm:objective-InfoMax}
Let \(\bm{X}_i = \{ \bm{x}_k \}_{k \in \mathcal{N}(i)}\) be the neighborhood of node \(v_i\) that collectively maps to its output embedding, where \(\mathcal{N}(i)\) denotes the set of neighbors of node \(v_i\) specified by GNN architectures, and \(\bm{X}\) be the corresponding random variable with a uniform distribution \(p(\bm{X}_i) = \nicefrac{1}{N}\). Given two random variables \(\bm{U, V} \in \mathbb{R}^{F^\prime}\) being the embedding in the two views, with their joint distribution denoted as \(p(\bm{U}, \bm{V})\), our objective \(\mathcal{J}\) is a lower bound of MI between encoder input \(\bm{X}\) and node representations in two graph views \(\bm{U, V}\). Formally,
\begin{equation}
	\mathcal{J} \leq I(\bm{X}; \bm{U}, \bm{V}).
\end{equation}
\end{theorem}

\begin{proof}
We first show the connection between our objective \(\mathcal{J}\) and the InfoNCE objective \cite{vandenOord:2018ut,Poole:2019vk} , which is defined as
\[I_\text{NCE}(\bm U; \bm V) \triangleq \mathbb{E}_{\prod_{i} p( {\bm u}_i, {\bm v}_i)} \left[ \frac{1}{N} \sum_{i=1}^N \log \frac{e^{\theta(\bm{u}_i, \bm{v}_i)}}{\frac{1}{N}\sum_{j = 1}^{N} e^{\theta(\bm{u}_i, \bm{v}_j)}} \right], \]
where the critic function is defined as \(\theta (\bm{x}, \bm{y}) = s(g(\bm{x}), g(\bm{y}))\).
We further define \(\rho_r({\bm{u}}_i) = \sum_{j \neq i}^N \exp(\theta({\bm u}_i, {\bm u}_j) / \tau)\) and \(\rho_c({\bm u}_i) = \sum_{j=1}^N \exp(\theta ({\bm u}_i, {\bm v}_j) / \tau)\) for convenience of notation.
\(\rho_r({\bm v}_i)\) and \(\rho_c({\bm v}_i)\) can be defined symmetrically.
Then, our objective \(\mathcal{J}\) can be rewritten as
\begin{equation}
	\mathcal{J} = \mathbb E_{\prod_{i} p( {\bm u}_i, {\bm v}_i)} \left[ \frac{1}{N} \sum_{i=1}^N \log \frac {\exp(\theta( {\bm u}_i, {\bm v}_i) / \tau)} {\sqrt{\left( \rho_c( {\bm u}_i) + \rho_r({\bm u}_i) \right) \left( \rho_c( {\bm v}_i) + \rho_r({\bm v}_i) \right)}} \right].
\end{equation}
Using the notation of \(\rho_c\), the InfoNCE estimator \(I_\text{NCE}\) can be written as
\begin{equation}
	I_\text{NCE}( {\bm U}, {\bm V}) = \mathbb E_{\prod_{i} p( {\bm u}_i, {\bm v}_i)} \left[ \frac{1}{N} \sum_{i=1}^N \log \frac {\exp(\theta( {\bm u}_i, {\bm v}_i) / \tau)} {\rho_{c}( {\bm u}_i)} \right].
\end{equation}
Therefore,
\begin{equation}
	\begin{aligned}
		2\mathcal{J} & = I_\text{NCE}(\bm U, \bm V) - \mathbb E_{\prod_i p( {\bm u}_i,  {\bm v}_i)} \left[ \frac{1}{N} \sum_{i=1}^N \log \left( 1 + \frac {\rho_r( {\bm u}_i)} {\rho_c( {\bm u}_i)} \right) \right] \\
		& \quad + I_\text{NCE}(\bm V, \bm U) - \mathbb E_{\prod_i p( {\bm u}_i, {\bm v}_i)} \left[ \frac{1}{N} \sum_{i=1}^N \log \left( 1 + \frac {\rho_r( {\bm v}_i)} {\rho_c( {\bm v}_i)} \right) \right] \\
		& \leq I_\text{NCE}(\bm U, \bm V) + I_\text{NCE}(\bm V, \bm U). \\
	\end{aligned}
\end{equation}
According to \citet{Poole:2019vk}, the InfoNCE estimator is a lower bound of the true MI, i.e.
\begin{equation}
	I_\text{NCE}(\bm{U}, \bm{V}) \le I(\bm{U}; \bm{V}).
\end{equation}
Thus, we arrive at
\begin{equation}
	 2\mathcal{J} \leq I(\bm U; \bm V) + I(\bm V; \bm U) = 2 I(\bm U; \bm V),
\end{equation}
which leads to the inequality
\begin{equation}
\label{eq:objective-uv}
\mathcal J \le I( {\bm U}; {\bm V}).
\end{equation}

According to the data processing inequality \cite{Cover:2006ei}, which states that, for all random variables \(\bm{X}, \bm{Y}, \bm{Z}\) satisfying the Markov relation \(\bm{X} \rightarrow \bm{Y} \rightarrow \bm{Z}\), the inequality \(I(\bm{X}; \bm{Z}) \leq I(\bm{X}; \bm{Y})\) holds.
Then, we observe that \(\bm{X}, \bm{U}, \bm{V}\) satisfy the relation \(\bm{U} \leftarrow \bm{X} \rightarrow \bm{V}\). Since, \(\bm{U}\) and \(\bm{V}\) are conditionally independent after observing \(\bm{X}\), the relation is Markov equivalent to \(\bm{U} \rightarrow \bm{X} \rightarrow \bm{V}\), which leads to \(I(\bm{U}; \bm{V}) \leq I(\bm{U}; \bm{X})\).
We further notice that the relation \(\bm{X} \rightarrow (\bm{U}, \bm{V}) \rightarrow \bm{U}\) holds, and hence it follows that \(I(\bm{X}; \bm{U}) \leq I(\bm{X}; \bm{U}, \bm{V})\).
Combining the two inequalities yields the required inequality
\begin{equation}
	\label{eq:data-processing}
	I(\bm U; \bm V) \leq I(\bm X; \bm U, \bm V).
\end{equation}
Following Eq. (\ref{eq:objective-uv}) and Eq. (\ref{eq:data-processing}), we finally arrive at inequality
\begin{equation}
	\mathcal{J} \leq I(\bm X; \bm U, \bm V),
\end{equation}
which concludes the proof.
\end{proof}

\subsection{Proof of Theorem 2}

\begin{theorem}
\label{thm:objective-triplet-loss}
When the projection function \(g\) is the identity function and we measure embedding similarity by simply taking inner product, and further assuming that positive pairs are far more aligned than negative pairs, i.e. \(\bm{u}_i^\top \bm{v}_k \ll \bm{u}_i^\top \bm{v}_i\) and \(\bm{u}_i^\top \bm{u}_k \ll \bm{u}_i^\top \bm{v}_i\), minimizing the pairwise objective \(\ell(\bm u_i, \bm v_i)\) coincides with maximizing the triplet loss, as given in the sequel
\begin{equation}
	\begin{split}
		- \ell (\bm u_i, \bm v_i) \propto 
		4N \tau + \sum_{j \neq i}\Bigg[ &\left(\| {\bm u_i} - {\bm v_i} \|^2 - \| {\bm u_i} - {\bm v_j} \|^2\right)\\
		&+ \left(\| {\bm u_i} - {\bm v_i} \|^2 - \| {\bm u_i} - {\bm u_j} \|^2\right) \Bigg].
	\end{split}
\end{equation}
\end{theorem}

\begin{proof}
Based on the assumptions, we can rearrange the pairwise objective as
\begin{equation}
	\begin{aligned}
		- \ell(\bm{u}_i, \bm{v}_i) & = - \log \frac {e^{\left( \bm{u}_i^\top \bm{v}_{i} / \tau\right)}} {\sum_{k=1}^{N} e^{\left( \bm{u}_i^\top \bm{v}_k / \tau\right)} + \sum_{k \neq i}^{N} e^{\left( \bm{u}_i^\top \bm{u}_k / \tau\right)}} \\
		& = \log \left( 1 + \sum_{k \neq i}^{N} e^{\left( \frac { {\bm u}_i^\top {\bm v}_k - {\bm u}_i^\top {\bm v}_i} {\tau} \right)} + \sum_{k \neq i}^{N} e^{\left( \frac { {\bm u}_i^\top {\bm u}_k - {\bm u}_i^\top {\bm v}_i} {\tau} \right)} \right).
	\end{aligned}
\end{equation}
By Taylor expansion of first order,
\begin{equation}
	\begin{aligned}
	& \quad - \ell( {\bm u}_i, {\bm v}_i) \\
	& \approx \sum_{k \neq i}^{N} \exp\left( \frac { {\bm u}_i^\top {\bm v}_k - {\bm u}_i^\top {\bm v}_i} {\tau} \right) + \sum_{k \neq i}^{N} \exp\left( \frac { {\bm u}_i^\top {\bm u}_k - {\bm u}_i^\top {\bm v}_i} {\tau} \right) \\
	& \approx 2 + \frac 1 \tau \left[ \sum_{k \neq i}^{N} (  {\bm u}_i^\top {\bm v}_k - {\bm u}_i^\top {\bm v}_i) + \sum_{k \neq i}^{N} ( {\bm u}_i^\top {\bm u}_k - {\bm u}_i^\top {\bm v}_i) \right] \\
	& = 2 - \frac 1 {2 \tau}\sum_{k \neq i}^{N} \left( \| {\bm u_i} - {\bm v_k} \|^2 - \| {\bm u_i} - {\bm v_i} \|^2 + \| {\bm u_i} - {\bm u_k} \|^2 - \| {\bm u_i} - {\bm v_i} \|^2 \right) \\
	& \propto 4N\tau + \sum_{k \neq i}^{N} \left(\| {\bm u_i} - {\bm v_i} \|^2 - \| {\bm u_i} - {\bm v_k} \|^2 + \| {\bm u_i} - {\bm v_i} \|^2 - \| {\bm u_i} - {\bm u_k} \|^2\right) ,
	\end{aligned}
\end{equation}
which concludes the proof.
\end{proof}

\bibliographystyle{ACM-Reference-Format}
\bibliography{www2021}

\end{document}